%% file: lgo.tex
\documentclass[twoside]{article}
\usepackage[accepted]{aistats2017}

\usepackage{amsmath,amssymb,amsthm,xcolor,graphicx,hyperref}
\usepackage{bm} % for mathbf
\usepackage{dsfont}
\usepackage{amsfonts}
\usepackage{calrsfs}
\usepackage{algorithm,algpseudocode}
\usepackage{aliascnt}  
\usepackage{color}
\usepackage{xspace}
\usepackage{multicol, blindtext}
\usepackage{mathrsfs}

\input{macros.tex}

\newcommand\enote[1]{{\textcolor{blue}{\textbf{Elad}: #1}}}
\newcommand\mnote[1]{{\textcolor{magenta}{\textbf{Mariano}: #1}}}
\newcommand\gnote[1]{{\textcolor{red}{\textbf{Gal}: #1}}}

\newcommand\anote[1]{{\textcolor{pink}{\textbf{Alan}: #1}}}
\renewcommand\enote[1]{{}}
\renewcommand\gnote[1]{{}}
\renewcommand\anote[1]{{}}

\begin{document}



\twocolumn[

\aistatstitle{Scalable Learning of Non-Decomposable Objectives}

\runningauthor{Elad Eban, Mariano Schain, Alan Mackey, Ariel Gordon, Rif A. Saurous, Gal Elidan}
\aistatsauthor{Elad Eban \\ \texttt{elade@google.com}%\\ Google Inc.
    \And Mariano Schain \\ \texttt{marianos@google.com}%\\ Google Inc.
    \And Alan Mackey \\ \texttt{mackeya@google.com} %\\ Google Inc.
    \AND Ariel Gordon \\ \texttt{gariel@google.com} %\\ Google Inc.
    \And Rif A. Saurous \\ \texttt{rif@google.com} %\\ Google Inc.
    \And Gal Elidan \\ \texttt{elidan@google.com} %\\ Google Inc.
    }

\aistatsaddress{\\ Google, Inc.}]

\begin{abstract}
Modern retrieval systems are often driven by an underlying machine learning model. The goal of such systems is to identify and possibly rank the few most relevant items for a given query or context. Thus, such systems are typically evaluated using a ranking-based performance metric such as the area under the precision-recall curve, the $F_\beta$ score, precision at fixed recall, etc. Obviously, it is desirable to train such systems to optimize the metric of interest.

In practice, due to the scalability limitations of existing approaches for optimizing such objectives, large-scale retrieval systems are instead trained to maximize classification accuracy, in the hope that performance as measured via the true objective will also be favorable. In this work we present a unified framework that, using straightforward building block bounds, allows for highly scalable optimization of a wide range of ranking-based objectives. We demonstrate the advantage of our approach on several real-life retrieval problems that are significantly larger than those considered in the literature, while achieving substantial improvement in performance over the accuracy-objective baseline.

\end{abstract}

\input{intro}
\input{related}
\input{basics}

\input{recall_at_precision.tex}
% Various AUC/AP
\input{auc.tex}
\input{f1_relax.tex}

\input{experiments.tex}
\input{summary.tex}
%%%%%%%%%%%%%%%%%%%%%%%%%%%%%%%%%%%%%%%%%%%%%%%%%%%%%%%%%%%%%%%%%%%%%%%%%%%%%%%%
\bibliographystyle{plain}
\bibliography{lgo}

\end{document}

%% file: macros.tex
\newaliascnt{lemma}{theorem}  
\newtheorem{lemma}[lemma]{Lemma}  
 
\theoremstyle{definition}
\newtheorem{definition}{Definition}[section]

\newcommand{\X}{X}

\newcommand{\E}{{\rm I\kern-.3em E}}
\newcommand{\reals}{\mathbb{R}}
\newcommand{\precision}{P}
\newcommand{\recall}{R}
\newcommand{\lowerbound}[1]{{#1}_{l}}
\newcommand{\upperbound}[1]{{#1}^{u}}
\newcommand{\fpu}{\upperbound{fp}}
\newcommand{\tpl}{\lowerbound{tp}}
\newcommand{\id}{\mathds{1}}

\newcommand{\aucpr}{\textbf{AUCPR}\xspace}
\newcommand{\aucroc}{\textbf{AUCROC}\xspace}
\newcommand{\patr}{\textbf{P@R}\xspace}
\newcommand{\ratp}{\textbf{R@P}\xspace}
\newcommand{\patrb}[1]{\textbf{P@R}{#1}}
\newcommand{\ratpa}[1]{\textbf{R@P}{#1}}
\newcommand{\map}{\textbf{MAP}\xspace}

\newcommand{\fb}{\mathbf{F_{\beta}}}

\newcommand{\posloss}{\mathscr{L}^{+}} 
\newcommand{\negloss}{\mathscr{L}^{-}} 

\renewcommand{\th}{b}

\newcommand\ignore[1]{}

\aliascntresetthe{lemma}

%% file: intro.tex
% !TEX root = lgo.tex
\section{Introduction}
\label{sec:intro}

Machine learning models underlie most modern automated retrieval systems. The quality of such systems is evaluated using ranking-based measures such as area under the ROC curve (\aucroc) or, as is more appropriate in the common scenario of few relevant items, measures such as area under the precision recall curve (\aucpr, also known as average precision), mean average precision (\map), precision at a fixed recall rate (\patr), etc. In fraud detection, for example, we would like to constrain the fraction of customers that are falsely identified as fraudsters, while maximizing the recall of true ones.

What is common to all of the above objectives is that, unlike standard classification loss, they only partially or do not at all decompose over examples. This makes optimization more difficult and consequently machine learning retrieval systems are often not trained to optimize the objective of
interest. Instead, they are typically trained simply to maximize classification accuracy in the hope that the retrieval performance will also be favorable. Unfortunately, this discrepancy can lead to inferior results, as is illustrated in \autoref{fig:example} (see also, for example, \cite{cortes2004auc,davis2006relationship,yue2007support}). Our goal in this work is to develop a unified approach that is applicable to a wide range of rank-based objectives and that is scalable to the largest of datasets, i.e. that is as scalable as methods that optimize for classification accuracy. 

\begin{figure}[!bh]
\centerline{
\includegraphics[width=0.9\columnwidth]{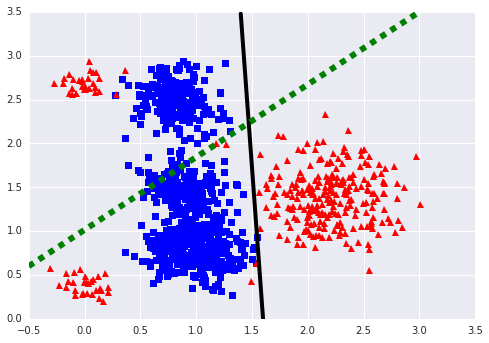}
}
\caption{Illustration of the potential difference between classification accuracy and, for example, the $\max\ \patr=0.95$ objective. The red triangles and blue squares represent positive and negative examples respectively. The black solid line corresponds to the best linear classifier which is ~90\% accurate. If the threshold of this classifier is changed to achieve recall of 95\%, the precision will be ~50\%. The dashed green line corresponds to a classifier that achieves a recall of 95\% but with precision of about 65\%.}
\label{fig:example}
\end{figure}

Several recent works, starting with the seminal work of Joachims \cite{joachims2005support} have addressed the challenge of optimizing various rank-based objectives. All of these, however, still suffer from computational scalability issues, or are limited to a specific metric. Joachims \cite{joachims2005support}, for example, offers a method for optimizing $\fb$ and \patr. In general, the computation of each gradient step is quadratic in the number of training instances, and slow even in the best of cases. As another example, Yue et al. \cite{yue2007support} optimize for \map but are hindered by the use of a costly cutting plane training algorithm. See \autoref{sec:related} for a more detailed list of related works and discussion of the scalability limitations.

In this work, we propose an alternative formulation that is based on simple direct bounds on per-example quantities indicating whether each example is a true positive or a false positive. These building block bounds allow us to construct global bounds on a wide range of ranking based, non-decomposable objectives, including all of those mentioned above. Importantly, the surrogate objectives we derive can be optimized using standard stochastic (mini-batch) gradient methods for saddle-point problems with favorable convergence rates \cite{chen2014optimal}. This is a decisive advantage at the massive scale on which most modern automated retrieval systems must operate, where methods requiring full-batch optimization are intractable. 

Following the development of the bounds for a range of measures, we demonstrate the effectiveness of our approach for optimizing \aucpr and other objectives on several real-life problems that are substantially larger than those considered in the literature. Empirically, we observe both improvement in performance compared to the use of standard loss functions (log-loss), and a favorable convergence rate indistinguishable from the vanilla SGD baseline rates.

Our contribution is thus threefold. First, we provide a unified approach that, using the same building blocks, allows for the optimization of a wide range of rank-based objectives that include \aucroc, \aucpr, \patr, \ratp, and $\fb$.  Third, our unified framework also easily allows for novel objectives such as the area under the curve for a region of interest, i.e. when the precision or recall are in some desired range. Finally, and most importantly, our bounds give rise to an optimization approach for non-decomposable learning metrics that is highly scalable and that is applicable to truly large datasets.

%% file: related.tex
% !TEX root = lgo.tex
\section{Related Works}
\label{sec:related}

Several works in the last decade have focused on developing methods
for optimizing rank-based objectives. Below we outline those most relevant to our work while highlighting the inherent scalability limitation which is our central motivation.

The seminal paper of Joachims \cite{joachims2005support} uses a bound on the possible number of contingency tables to optimize $\fb$ and \patr, and a bound
based on individual pairs of examples to optimize \aucroc; in both cases the system iteratively solves a polynomial-time optimization sub-problem to find a constraint to add to a global optimization, which generalizes structured SVMs
\cite{tsochantaridis2005large}. The scalability of this approach is limited since the cost of computing a single gradient is generally quadratic (and always super-linear) in the number of training examples. Furthermore, even in the best case, Joachims' loss function and its gradient take at least linear time to compute, resulting in a slow gradient-descent algorithm. This is in contrast to our stochastic gradient approach.

Optimizing the \aucpr\ or the related mean-average-precision is, in principle, even more difficult since the objective does not decompose over pairs of examples. \cite{metzler2005markov} and \cite{caruana2006empirical}
tackle this objective directly, and \cite{yue2007support} proposes a
more efficient AP-SVM which relies on a hinge-loss relaxation of the
\map problem.
%\enote{gal: what do you think about reverting rif's comment-out of %cortes2004auc}
% Note: I removed the ref to cortes2004auc here, since it doesn't
% address this problem at all --- it provides an *analysis* of AUC
% vs. accuracy. I added the ref to ``this discrepancy can lead to
% inferior results'' above. I removed the reference to
% rosen2014learning because it didn't seem *that* relevant [it's AUROC
% for structured problems], and I couldn't quite fit it.
%
While the work of \cite{yue2007support} demonstrates the merit of optimizing \map instead of accuracy for reasonably sized domains, scalability is still hindered by the use of a cutting plane training algorithm that requires a costly identification of the most violated constraint. To overcome this, \cite{mohapatra2014efficient} suggests several innovative heuristic improvements that achieve appealing running time gains, but do not inherently solve the underlying scalability problem. \cite{song2015direct} generalizes the above approaches to the case of nonlinear deep network optimization; their approach ``has the same complexity'' as \cite{yue2007support}, and is thus also not scalable to very large problems.

Other approaches achieve scalability by considering a restricted class of models \cite{metzler2005direct} or targeting only specific objectives  \cite{chase2014thresholding, herschtal2004optimising,nan2012optimizing,  parambath2014optimizing,rakotomamonjy2004optimizing}. The work of \cite{quoc2007learning} achieves both scalability and generality, but does not cover objectives which place a constraint on the model, such as recall at a fixed precision, precision at a fixed recall, or accuracy at a quantile. \cite{boyd2012accuracy} optimize this latter objective in an elegant method that is theoretically scalable, since  it can be distributed to many machines. However, it requires solving an optimization problem per instance and is thus not truly scalable in practice. Finally, \cite{kar2014online} propose a general purpose approach for optimizing non-decomposable objectives. Their method is cast in the online setting, and the adaptation they suggest for stochastic gradient optimization with minibatches requires a buffer. For extremely multi-label problems such as one we consider below, maintaining such a buffer may not be possible.

%% file: basics.tex
% !TEX root = lgo.tex
\section{Building Block Bounds}
\label{sec:building}
In this section we briefly describe the simple building block bounds
of the true positive and false positive quantities. These statistics
will form the basis for the objectives of interest throughout our
work. 

We start by defining the basic entities involved in rank-based
metrics. We use $X$ to denote the explanatory features, $Y$ to
denote the target label, $Y^+$ to denote the positive examples, and 
$Y^-$ to denote the negatives.

\begin{definition}
A classification rule $f_{\th}$ is characterized by a score function 
$f: \X \rightarrow \reals$, and a threshold $\th \in \reals$,
indicating that classification is done according to $f(x) \ge \th$.
\end{definition}
Note that we intentionally separate the parameters of the models
embedded in $f$ (which could be a linear model or a deep neural-net), and the
decision threshold $\th$. The former provides a score which defines a
ranking over examples, while the latter defines a decision boundary on
the score that separates examples that are predicted to be relevant
(positive) from those that are not. 

\begin{definition} 
The precision $\precision(f_{\th})$ and recall $\recall(f_{\th})$ of a 
classification rule are defined by:
\begin{equation*}
\begin{split}
  \precision(f_{\th}) = &\frac{tp(f_{\th})}{tp(f_{\th}) +fp(f_{\th})} \\ %&\qquad \qquad
  \recall(f_{\th}) = &\frac{tp(f_{\th})}{tp(f_{\th}) +fn(f_{\th})} = \frac{tp(f_{\th})} {|Y^+|}
\end{split}\end{equation*}
where $tp,fp,fn$ are the true-positives, false-positives, and 
false-negative counts (respectively):
\[
  tp(f_\th) =  \sum_{i \in Y^+} \id_{f(x_i) \ge \th} \qquad %\qquad 
  fp(f_\th) =  \sum_{i \in Y^-} \id_{f(x_i) \ge \th }
\]
\end{definition}

We lower bound $tp$ and upper bound $fp$ by first writing them in terms of the zero-one loss:
\begin{equation}
\begin{split}
  tp(f_\th) =&  \sum_{i \in Y^+} 1-\ell_{01}(f_b,x_i,y_i) \\  
  fp(f_\th) =&  \sum_{i \in Y^-} \ell_{01}(f_b,x_i,y_i) 
\end{split}
\end{equation}
we do not need to
bound $fn$ because it shows up only in the denominator of the
expression for recall and can be eliminated via $|Y^+| = tp + fn$.
Now it is natural to bound these quantities by using a surrogate for the zero-one loss function such as the hinge loss:
\begin{equation}
\label{tp_tl_bounds}
\begin{split}
 \tpl(f_\th) \triangleq \;&   \sum_{i \in Y^+}1-\ell_h(f_{\th},x_i,y_i) \leq tp(f_\th), \\
 \fpu(f_\th) \triangleq \;& \sum_{i \in Y^-}\ell_h(f_{\th},x_i,y_i) \geq fp(f_\th) ,
\end{split}
\end{equation}
where
\begin{equation*}
\ell_h(f_{\th},x,y) \triangleq \max(0, 1 -y (f(x) - \th))
\end{equation*}
is the hinge loss of the score $f(x) - \th$ on point $x$ with label
$y\in \{-1,1\}$. The right-hand inequalities follow directly. 
We note that in what follows we use the hinge-loss for simplicity but other losses could be used in all the results presented below (with the exception of the linear-fractional transformation of the $F_\beta$ score in \autoref{sec:f1}). In the case of convex surrogates for the zero-one loss such as the log-loss or the smooth-hinge-loss \cite{rennie2005smooth} we get convex (and smooth) optimization problems. However, our method can also be used with non-convex surrogates such as the ramp-loss.

These simple bounds will allow us to bound a variety of global
non-decomposable ranking measures including the \aucroc, \aucpr,
$F_\beta$, etc.
%
% \footnote{We note that almost all our results hold for any convex % upper-bound for the 0-1 loss, such as the log-loss,
%which is used in our experiments. The exception is the %linear-fractional
%transformation of the $F_{\beta}$ score in \autoref{sec:f1}~.}

%% file: recall_at_precision.tex
% !TEX root = lgo.tex
\section{Maximizing Recall at Fixed Precision}\label{sec:recall@precistion}

In this section we show how the building block bounds of
\eqref{tp_tl_bounds} can provide a concave lower bound on the
objective of maximum recall with \emph{at least} $\alpha$ precision.
A similar derivation could also be used to provide a bound on maximum
precision given a minimum desired recall. Aside from the stand-alone 
usefulness of the \patr and \ratp metrics, the developments here will
underlie the construction for optimizing the maximum \aucpr\ objective
that we present in the next section.

We begin by defining the maximum recall at fixed minimum precision problem:
\begin{equation}\label{recall_at_precision}\begin{split}
     \ratp\alpha = \quad \max_{f} &\quad \recall(f) \\
     s.t. &\quad \precision(f) \ge \alpha
\end{split}\end{equation}

The above is a difficult combinatorial problem. Thus, instead of
solving it directly, we optimize a lower bound similarly to how the
hinge loss is used as a surrogate for accuracy in SVM optimization
\cite{cortes1995support}. To do so, we write
\eqref{recall_at_precision} as
\begin{equation*}
\begin{split}
     \max_{f, \th} &\quad \frac{1}{|Y^+|} tp(f) \\
     s.t. &\quad tp(f) \ge \alpha(tp(f) + fp(f)).
 \end{split}
\end{equation*}
To turn this objective into a tractable optimization surrogate, we use
\eqref{tp_tl_bounds} to lower bound $tp$ and upper bound $fp$:
\begin{equation}
\label{surrogate_recall_at_precision}
\begin{split}
     \overline{\ratp\alpha} = \quad \max_{f, \th} &\quad \frac{1}{|Y^+|}\tpl(f) \\
     s.t. &\quad   (1- \alpha)\tpl(f) \ge \alpha \fpu(f).
\end{split}
\end{equation}
\begin{lemma}
\label{lem:concavelower}
The relaxed problem $\overline{\ratp\alpha}$ is a concave lower bound 
for $\ratp\alpha$.
\end{lemma}

\begin{proof}
To see that the surrogate problem is a lower bound of the original problem  we notice that the surrogate recall $\frac{1}{|Y^+|} \tpl(f)$ is a lower bound on the true recall. In addition we notice that the surrogate precision  $\overline{\precision} = \frac{\tpl(f)}{\tpl(f) + \fpu(f)}$ is a lower bound on the actual precision. Hence the feasible set of $\overline{\ratp\alpha}$ is contained in the feasible set of $\ratp\alpha$,  as
\[        
\precision(f) \geq \overline{\precision}(f)   \geq \alpha. 
\]
This proves that the surrogate problem is a lower bound on the original problem.

Finally, the objective of $\overline{\ratp\alpha}$ is concave,
and the constraints are convex (in fact they are piece-wise linear).
\end{proof}

The relaxed objective of \eqref{surrogate_recall_at_precision} is now
amenable to efficient optimization. To see this, we plug the explicit
forms of $\tpl(f)$ and $\fpu(f)$ into
\eqref{surrogate_recall_at_precision}:

\begin{equation*}
\begin{split}
    \overline{\ratp\alpha} =  \max_{f} &\quad  1 - \frac{\posloss(f)}{|Y^+|} 
     \\
    s.t. &\; (1- \alpha)(|Y^+|-\posloss(f)) \ge 
        \alpha\negloss(f).
\end{split}
\end{equation*}
Where we use as a shorthand 
\[
\posloss(f) = \sum_{i\in Y^+}\ell_h(f,x_i,y_i) 
\] 
for the loss on the positive examples, and similarly $\negloss$ is the sum of errors on the negative examples. We omit the explicit dependence on $f$ when it is clear from context. 
Next, we rewrite the constraint and ignore the constant multiplier in the objective to
obtain the following equivalent (with respect to the optimal $f, \th$) problem:
\begin{equation}
\label{eq:minp@r}
\begin{split}
\min_{f} &\;  \posloss \\
  s.t. &\;  \alpha \negloss  
               +   (1-\alpha)\posloss \le   (1-\alpha)|Y^+|.  
\end{split}
\end{equation}
Now, applying Lagrange multiplier theory, we can equivalently consider the following objective:
\begin{equation*}
\begin{split}
\min_{f} \max_{\lambda \ge 0}  \quad  &
  \posloss   % \\   & 
  +\lambda \left( \frac{\alpha}{1-\alpha} \negloss
               + \posloss  - |Y^+| \right).
\end{split}
\end{equation*}
Finally, after some regrouping of terms, this can be written as:
\begin{equation}
\begin{split}
\label{eq:dual} 
\min_{f} \max_{\lambda \ge 0} &\quad  
 (1+\lambda)\posloss   
 + \lambda \frac{\alpha}{1-\alpha}  \negloss
 - \lambda|Y^+| . 
\end{split}
\end{equation}
We now face a saddle point problem, which we optimize using the following straightforward iterative stochastic
gradient descent (SGD) updates:
\begin{equation*}
\begin{split}
f^{(t+1)} = 
    & f^{(t)} - \gamma \nabla L(f^{(t)},\lambda^{(t)}) \\
\lambda^{(t+1)} = 
    & \lambda^{(t)} + \gamma \nabla L(f^{(t+1)},\lambda^{(t)})
\end{split}
\end{equation*}
where 
\[
L(f,\lambda) = (1+\lambda)\posloss(f)   
    + \lambda \frac{\alpha}{1-\alpha}  \negloss(f)
    - \lambda|Y^+|.
\]
\begin{lemma}
The above procedure converges to a fixed point if both $\posloss,\negloss$ are convex.
\end{lemma}
The proof is straightforward can be found in \cite[Section 3]{nedic2009subgradient}.

Aside from the obvious appeal of algorithmic simplicity, it is
interesting to note that the above objective supports the standard
practice of trying to achieve good \patr or \ratp via example
re-weighting. To see this, note that for a fixed $\lambda$, the
minimization over $f$ in \autoref{eq:dual} is just a $c(\alpha,
\lambda)$ weighted SVM. Specifically, after adding a regularization term, the SVM objective takes the form
\begin{equation}\begin{split}\label{eq:wsvm} 
\min_{f} &\quad  
  \sum_{i}\ell_h^c(f,x_i,y_i) + \|f\|^2,
\end{split}\end{equation}
where $\ell_h^c$ is the loss when a positive instance is weighted
by $c(\alpha, \lambda) = \frac{(1+\lambda)(1-\alpha)}{\lambda\alpha}$.
Because $c(\alpha, \lambda)$ is monotonic in $\lambda$, the
problem can also be solved via a binary search for this single dual parameter.

\subsection{Maximizing \patr}
Following the same steps we can reach the following optimization problem:
\begin{equation}\label{patr}
\overline{\patr}\beta = \min_f\max_{\lambda \geq 0} \frac{-|Y^+|\beta}{|Y^+|\beta + \fpu(f)} - \lambda \left(\frac{\tpl(f)}{|Y^+|} - \beta\right)
\end{equation}
which seems odd at first as we expect it also to result in a re-weighting of the positives and negatives (mediated by $\lambda$) as in  $\ratp$. However, an equivalent problem which minimizes $1/\precision$ rather than $-\precision$ achieves the expected formulation
\begin{equation}
\overline{\patr}\beta = \min_f\max_\lambda  \negloss 
+ \lambda\left(\beta +\frac{\posloss}{|Y^+|} -1 \right).
\end{equation}
We note that this problem has the same minimizer as \autoref{patr}, but not the same value.

%% file: auc.tex
% !TEX root = lgo.tex
%%%%%%%%%%%%%%%%%%%%%%%%%%%%%%%%%%%%%%%%%%%%%%%%%%%%%%%%%%%%%%%%%%%%%%%%%%%%%%%%
\section{Maximizing \aucpr}
\label{sec:auc}
We are now ready to use our derivation of the R@P optimization objective
in the previous section order to construct a concave lower-bound surrogate
for \aucpr.  A similar derivation could be used for \aucroc optimization.

To start, recall that \aucpr is simply an integral over \ratp (equivalently \patr) values. That is:
\begin{equation}\label{eq:apdef}
\aucpr(f) = \max_f \int_{\pi}^1 \recall@\precision\alpha(f) d\alpha,
\end{equation}
where $\pi$ is the positive class prior, and $\recall@\precision\alpha(f)$ denotes
the recall we achieve when using $f$ as a score function with $b=b(\alpha)$ is a threshold which achieves precision $\alpha$. Another way to think of $b$ is via the optimization problem 
\begin{equation*}
\begin{split}
\recall@\precision\alpha(f) = & \max_{\th} \recall(f_{\th}) \\ &\text{s.t. } \precision(f_{\th}) = \alpha.
\end{split}
\end{equation*}

To apply our bounds to the objective of maximizing $\aucpr(f)$, we first approximate the integral in
\autoref{eq:apdef} by a discrete sum over a set of precision anchor values $A=\{\pi=\alpha_0 <
\alpha_1 < \alpha_2 < \ldots < \alpha_k\}$:
\begin{equation}
\begin{split}
\max_f & \; \aucpr(f) =  \\  \max_f \; & \sum_{t=1}^k \Delta_t \left[ \max_{\th_t} \recall(f_{\th_t}) \quad \text{s.t. } \precision(f_{\th_t}) \geq \alpha_t \right],
\end{split}
\end{equation}
where 
\begin{equation*}
 \Delta_t = \alpha_t-\alpha_{t-1} \qquad \forall t=1 \ldots k.
\end{equation*}
Naturally, one could take uniformly spaced $\alpha$ and set $\alpha_t = \pi + \tfrac{(1-\pi)t}{k}$, though this is not required.

Next, using the same technique we used for the maximum \ratp objective, we 
relax the building block statistics and, after some algebraic
manipulations and application of the Lagrange multiplier theory, we get:
\begin{equation}\begin{split}
\min_{f,\th_1,\ldots\th_k}  \max_{\lambda_1\ldots\lambda_k}
\sum_{t=1}^k \Delta_t \Big(&  (1+\lambda_t)\posloss(f,b_t)  
 \\ &
 + \lambda_t \frac{\alpha_t}{1-\alpha_t}  \negloss(f,b_t) 
 - \lambda_t|Y^+| \Big).
\end{split}\end{equation} 
As before, we can solve this saddle point problem by SGD \cite{nedic2009subgradient}.

By replacing \ratp with true positive rate at fixed false positive rate in the derivation above, we obtain a similar algorithm for optimizing \aucroc. The building block bounds from Section \ref{sec:building} can be used to generate surrogate objectives for true positive rate at the false positive rate anchor points. 

An important consequence of the above derivation is that we can just
as easily optimize for \aucroc and \aucpr in some limited range,
e.g. for precision greater than some desired threshold. This would
amount to constraining the range of precision anchor values in the
above development, and can be optimized just as easily.

%% file: f1_relax.tex
\section{Optimizing the $F_{\beta}$ Measure} 
\label{sec:f1}

To demonstrate the flexibility of our unifying framework, we now show
that the building block bounds of \eqref{tp_tl_bounds}, with a few
additional manipulations, can also be used to optimize the commonly
used $F_{\beta}$ score. This score
is a measure of the effectiveness of retrieval with respect to a user
who attaches $\beta$ times as much importance to recall as precision
\cite{van1979information}. The $F_{\beta}$ score is defined as:
\[
F_{\beta} =  (1+\beta^2)\frac{\precision \cdot \recall }{\beta^2 \precision + \recall}
\]

On the surface it is not clear how to use bounds on the $tp$ and $fp$ to bound $F_{\beta}$
since these statistics appear both in the denominator and numerator. We can do so by first
rewriting the $F_{\beta}$ score in the well known type I and type II error form:
\begin{equation*}\begin{split}
F_{\beta} = &(1+\beta^2) \frac{tp}{(1+\beta^2) tp + \beta^2 fn +fp} \\
= &(1+\beta^2)\frac{tp}{\beta^2|Y^+| +tp +fp}
\end{split}.
\end{equation*}

We now plug in the bounds from \autoref{tp_tl_bounds}, and get a surrogate function for $F_{\beta}$:
\begin{equation*}
 \overline{F_{\beta}} = (1+\beta^2)\frac{\tpl}{\beta^2|Y^+| +\tpl + \fpu} \leq F_{\beta}.
\end{equation*}
The lower bound follows from the fact that subtracting the same quantity (the difference between
$tp$ and $\tpl$) from both the numerator and denominator decreases the quotient, and that increasing
the denominator (by replacing $fp$ with $\fpu$) also decreases it.

%The above is a linear-fractional form and using a standard transformation technique
%%\cite{boyd2004convex} it is easy to show that the $F_1$ optimization task is equivalent to:
% \begin{equation}\label{eq:f1equiv}\begin{split}
% \min_{\omega,\epsilon\ge 0} \max_{\lambda}&\quad 
%     (1- \lambda) \epsilon \pi
%     + \tfrac{1}{m}\sum_{i}^m \ell_h^{\lambda,\epsilon}(\omega,x_i,y_i)
%     + \lambda 
% \end{split}\end{equation}

Our goal now is to maximize the above lower bound efficiently. For simplicity we demonstrate this for $F_1$
but the details are essentially the same for $F_{\beta}$. First, we note that maximizing $\overline{F_1}$ is 
equivalent to minimizing $(\overline{F_1})^{-1}$, and write this objective as a fractional linear program \cite{boyd2004convex}:
\begin{equation*}\label{eq:f1frac}
\begin{split}
\min_{f,t,w} &\quad \frac{|Y^+| +\sum_{i \in Y^+} t_i +\sum_{i \in Y^-} f_i}{\sum_{\forall {i \in Y^+}	} t_i}\\
s.t.& \\
\forall {i \in Y^+}		&\quad t_i \le 1 ,\quad t_i \le w \cdot x_i \\
\forall {i \in Y^-}		&\quad f_i \ge 0,\quad f_i \ge 1 + w \cdot x_i.
\end{split}
\end{equation*}
We can now use the linear-fraction transformation \cite{boyd2004convex} to derive the equivalent problem in variables $\tau,\phi,\omega$ and $\epsilon$:
 \begin{equation}\label{eq:f1equiv}\begin{split}
 \min_{\phi,\tau,\omega,\epsilon} &\quad 
     |Y^+|\epsilon +\sum_{ {i \in Y^+}} \tau_i +\sum_{ {i \in Y^-}	} \phi_i\\
 s.t.& \\
\forall {i \in Y^+}		&\quad \tau_i \le \epsilon  ,\quad
 \tau_i \le \omega \cdot x_i \\
\forall {i \in Y^-}		&\quad \phi_i \ge 0 ,\quad \phi_i \ge \epsilon + \omega \cdot x_i \\
 &\quad \sum_{i\in Y^+} \tau_i = 1\\
 &\quad \epsilon \ge 0,
 \end{split}\end{equation}
where we used the mappings $\omega = \epsilon w$, $\phi_i = \epsilon f_i$, $\tau_i = \epsilon t_i$ and $\epsilon = \frac{1}{\sum_i t_i}$. The resulting linear program can be of course solved in various ways, e.g. using an iterative gradient ascent
procedure as with the \ratp and \aucpr\ objectives.

Alternatively, the task of maximizing the $F_\beta$ measure can be solved using a constrained optimization approach.
First we write the minimization of $(F_1)^{-1}$ task as a function of $\posloss,\negloss$:
 
 \begin{equation}\begin{split}
 \min (F_1)^{-1} =& \min_{f} \frac{(|Y^+| + |Y^+| -\posloss+\negloss)}{(|Y^+| -\posloss)}, \\
  \end{split}\end{equation}
which is equivalent to
\begin{equation}\begin{split}
\min_{f,\psi} 
 & \frac{(|Y^+| + \psi+\negloss)}{\psi}\\
 s.t. & \; \psi = |Y^+| -\posloss
\end{split}\end{equation}

after defining an auxiliary variable $\psi = (|Y^+| -\posloss)$. Some simple algebra shows that the above is equivalent to

 \begin{equation}\begin{split}
 \min_{f,\psi} \max_\lambda 
 \psi^{-1}\negloss  
 + \lambda \posloss
 + (\psi^{-1}-\lambda)|Y^+| 
 + \psi\lambda.
 \end{split}\end{equation} 
From this formulation we see again that given $\psi,\lambda$ we have a weighted classification problem as was proved by a different technique in \cite{parambath2014optimizing}.

%% file: experiments.tex
% !TEX root = lgo.tex
\section{Experimental Evaluation}

In this section we demonstrate the merit of our approach for learning with a non-decomposable objective. We focus mostly on the optimization of the \aucpr objective due to its wide popularity in ranking scenarios and, as discussed, the fact that existing methods for optimizing this metric are not sufficiently scalable. With this in mind we consider three challenging problems, two of which are substantially (by orders of magnitude) larger than the those considered in the literature (e.g. \cite{song2015direct}).

\ignore{I wanted to introduce some empirical notes here
\begin{itemize}
\item dual step size 
\item augmented Lagrangian
\item releasing code ?
\item 
\end{itemize}
}

\subsection{CIFAR-10}
The CIFAR-10 dataset consists of 60000 32x32 color images in 10
classes, with 6000 images per class. The goal is to disinguish between the 10 classes. There are 50000 training images
and 10000 test images.  As our baseline we use the deep
convolutional network from \href{https://github.com/tensorflow/tensorflow/tree/r0.11/tensorflow/models/image/cifar10/}{TensorFlow.org}. Both the baseline model (trained with a soft-max loss) and our model (trained with the \aucpr loss) were optimized for 200,000 SGD steps with 128 images per batch. All the models  were trained on a single tesla k40 GPU for about eight hours.
Recall that our method relies on $K$ discrete anchor points to approximate the \aucpr integral. To evaluate the
robustness for this choice, we learned three models with 5, 10 and 20 points. Results were essentially identical for all of these settings, so for brevity only results with $K = 10$ are reported below. We also compared to the standard pairwise \aucroc surrogate \cite{rakotomamonjy2004optimizing}.

\ignore{
The results are summarized in the following table:
%
%\twocolumn[
\begin{figure*}[!th]
\begin{center}
\begin{tabular}{||l | c | c | c | c | c |c||} \hline
\hbox{Objective vs. Metric} & \aucpr &	\patrb{70} &	\patrb{90} & 	\ratpa{70} & 	\ratpa{90} &	Accuracy  \\ \hline \hline
\aucpr  &	94.2\% &	95.3\% &	82.5\% &	93.5\% &	83.5\% &	87.2\%  \\ \hline
\patrb{70} &	94.2\% &	95.5\% &	82.9\% &	93.9\% &	83.4\% &	87.0\% \\ \hline
\patrb{90} &	94.3\% &	95.5\% &	82.9\% &	94.0\% &	83.8\% &	87.3\% \\ \hline
\ratpa{70} &	94.1\% &	95.3\% &	81.9\% &	93.8\% &	83.2\% &	86.5\% \\ \hline
\ratpa{90} &	94.4\% &	95.7\% &	83.5\% &	94.0\% &	84.3\% &	87.6\% \\ \hline
SoftMax &	84.6\% &	82.0\% &	60.0\% &	81.1\% &	54.5\% 	& 87.1\% \\ \hline
\aucroc	 &94.2\% &	95.2\% &	82.1\% &	94.0\% &	83.3\% &	87.0\%\\ \hline
\end{tabular}
\caption{This table contains CIFAR10 test-set performance of different algorithm with respect to different performance metrics. Rows contain performance of a loss function over various metrics. It is apparent that all ranking based metrics outperform the standard soft-max algorithm by 10-30\%. Although \ratpa{90} loss function achieves best performance in most metrics, comparing our various methods to \aucroc, or soft-max based optimization is always favorable for our objectives.}
\end{center}
\end{figure*}
}

The advantage of optimizing for the objective of interest is clear: optimizing for \aucpr rather than accuracy increases the \aucpr metric from 84.6\% to 94.2\%. Results for other metrics are presented below in \autoref{fig:gaintable}.

\begin{table}[!bh]
\centerline{
\begin{tabular}{||l | c | c ||} \hline
Metric \textbackslash Gain over &  \aucroc &  softmax  \\ 
\hline \aucpr      &  0.0\%&	9.6\%\\
\hline \patrb{70}  &  0.2\%& 13.5\%  \\
\hline \patrb{95}  &  0.8\%& 24.1\% \\
\hline \ratpa{70}  & -0.2\%& 12.7\% \\
\hline \ratpa{95}  &  1.9\%&  36.6\% \\
\hline
\end{tabular}
}
\caption{The gain over a baseline loss (in absolute percentage points) in various metrics when optimizing for that metric with our framework. The baseline losses considered are pairwise \aucroc and standard softmax cross-entropy. In all cases, the model architectures being optimized are identical.}
\label{fig:gaintable}
\end{table}

To get a more refined view of the differences between the baseline model and our approach, \autoref{fig:cifarcv} shows the aggregate precision-recall curve across all classes (left), and the per-class breakdown (right). The aggregate difference between the two models is evident and the advantage of our method in the interesting range of high precision is particularly impressive. Looking at the per-class performance, we see that our approach improves the \aucpr for \emph{all} 10 classes, and substantially so for classes where the performance of the baseline is poor.

%gal: took out, will only confuse the reader
%Another interesting point from these experiments is that the accuracy %of ranking based models is usually better than the soft-max models. As %the train performance in this simple dataset approaches 1.0 for all %algorithm, we speculate that the ranking measures provide some %\emph{regularization} effect which helps generalization. 

\begin{figure*}[!th]
\begin{center}
\begin{tabular}{cc}
\hspace{-0.4in} 
 \includegraphics[width=0.95\columnwidth]{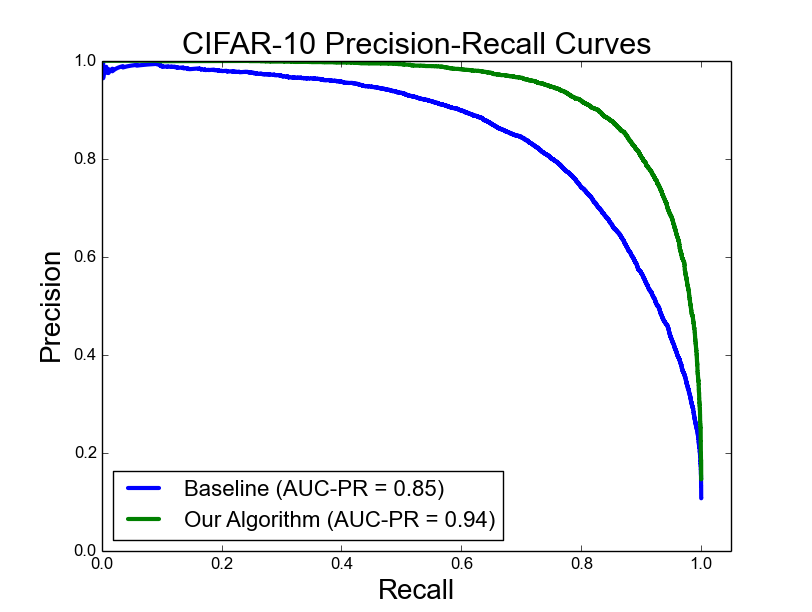} 
 \hspace{-0.5in} &
\includegraphics[width=0.95\columnwidth]{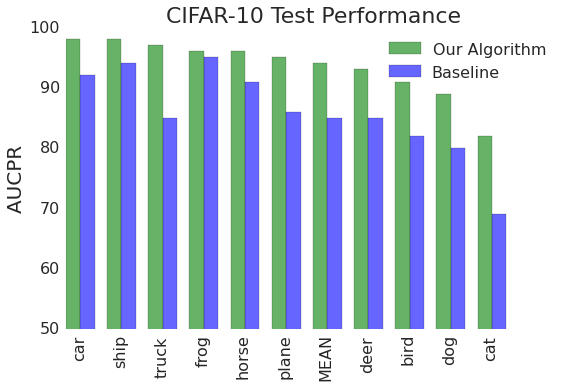}
\end{tabular}
\caption{Comparison of a baseline model trained to optimize accuracy
and a model trained to optimize \aucpr using our method on the CIFAR10 dataset. (left) Shows the aggregate precision-recall curve for all 10 classes. (right) Compares the \aucpr for each of the 10 classes.}
\label{fig:cifarcv}
\end{center}
\end{figure*}

\subsection{ImageNet}
The ILSVRC 2012 image dataset \cite{ILSVRC15} contains 1.2 million
images for training and 50K for validation. The goal is to distinguish
between 1000 object classes. This dataset, also known as ImageNet, is
used as the basis for competitions where the accuracy at the top 1 or
5 predictions is measured. We used the same dataset to demonstrate
that we can trade off accuracy and \aucpr\ at scale. We use the
Inception-v3 open source implementation from TensorFlow.org
\cite{abadi2016tensorflow} as our baseline. 
The ImageNet experiments were trained with 50 tesla K40 GPU replicas for three days performing about 5M mini-batch updates. Using the same
architecture, we optimize for \aucpr, allowing both the baseline and
our method the same training time. By using $K=5$ anchor 
points to approximate the \aucpr\ integral, we increase the \aucpr\ from
82.2\% for the baseline to 83.3\%, while decreasing accuracy by
0.4\%. We note that, generally speaking, improvements on the order of
1\% for ImageNet are considered substantial.

\subsection{JFT}
To demonstrate merit and applicability of our method on a truly 
large-scale domain, we consider the Google internal JFT dataset.
This dataset has over 300 million labeled images with about $20,000$ labels. As our baseline we use a deep convolutional neural network
which is based on the Inception architecture \cite{szegedy2015going}. 
The specific architecture used is Google's current state-of-the-art model. Performance of models on this data is evaluated first and foremost using the \aucpr\ metric, yet the baseline model is trained to maximize accuracy via a logistic loss function. To learn
a model using our approach, we start training from the pre-trained parameters (training from scratch is a multi-month process), and optimize the \aucpr\ objective for several days. To have a fair comparison,
we also allow the baseline model to continue training for the same
amount of time. While the baseline model achieves 
an \aucpr\ of 42\%, the model optimized with our surrogate achieves an \aucpr\
of 48\%, a substantial improvement.

%% file: summary.tex
% !TEX root = lgo.tex
\section{Summary and Future Directions}

In this work we addressed the challenge of scalable optimization non-decomposable ranking-based objective. We introduced simple building block bounds that provide a unified framework for efficient optimization of a wide range of such objectives. We demonstrated the empirical effectiveness of our approach on several real-life datasets.

Importantly, some of the problems we consider are dramatically larger than those previously considered in the literature. In fact, our approach is essentially as efficient as optimization of the fully decomposable accuracy loss. Indeed, our method can be coupled with any shallow (SVM, logistic regression) or deep (CNN) architecture with negligible cost on performance. 

Aside from the obvious appeal of scalability, our unified approach also opens the door for novel optimization of more refined objectives. For example, maximizing the area under the ROC curve for a pre-specified range of the false-positive rate is as easy as maximizing the area under the precision recall curve altogether. In future work, we plan to explore the importance of such flexibility for real-world ranking problems. \gnote{Want to mention random forests here?}

%% file: lgo.bbl
\begin{thebibliography}{10}

\bibitem{abadi2016tensorflow}
Mart{\i}n Abadi, Ashish Agarwal, Paul Barham, Eugene Brevdo, Zhifeng Chen,
  Craig Citro, Greg~S Corrado, Andy Davis, Jeffrey Dean, Matthieu Devin, et~al.
\newblock Tensorflow: Large-scale machine learning on heterogeneous distributed
  systems.
\newblock {\em arXiv preprint arXiv:1603.04467}, 2016.

\bibitem{boyd2012accuracy}
Stephen Boyd, Corinna Cortes, Mehryar Mohri, and Ana Radovanovic.
\newblock Accuracy at the top.
\newblock In {\em Advances in neural information processing systems}, pages
  953--961, 2012.

\bibitem{boyd2004convex}
Stephen Boyd and Lieven Vandenberghe.
\newblock {\em Convex optimization}.
\newblock Cambridge university press, 2004.

\bibitem{caruana2006empirical}
Rich Caruana and Alexandru Niculescu-Mizil.
\newblock An empirical comparison of supervised learning algorithms.
\newblock In {\em Proceedings of the 23rd international conference on Machine
  learning}, pages 161--168. ACM, 2006.

\bibitem{chase2014thresholding}
Zachary Chase~Lipton, Charles Elkan, and Balakrishnan Narayanaswamy.
\newblock Thresholding classifiers to maximize f1 score.
\newblock {\em arXiv preprint arXiv:1402.1892}, 2014.

\bibitem{chen2014optimal}
Yunmei Chen, Guanghui Lan, and Yuyuan Ouyang.
\newblock Optimal primal-dual methods for a class of saddle point problems.
\newblock {\em SIAM Journal on Optimization}, 24(4):1779--1814, 2014.

\bibitem{cortes2004auc}
Corinna Cortes and Mehryar Mohri.
\newblock Auc optimization vs. error rate minimization.
\newblock {\em Advances in neural information processing systems},
  16(16):313--320, 2004.

\bibitem{cortes1995support}
Corinna Cortes and Vladimir Vapnik.
\newblock Support-vector networks.
\newblock {\em Machine learning}, 20(3):273--297, 1995.

\bibitem{davis2006relationship}
Jesse Davis and Mark Goadrich.
\newblock The relationship between precision-recall and roc curves.
\newblock In {\em Proceedings of the 23rd international conference on Machine
  learning}, pages 233--240. ACM, 2006.

\bibitem{herschtal2004optimising}
Alan Herschtal and Bhavani Raskutti.
\newblock Optimising area under the roc curve using gradient descent.
\newblock In {\em Proceedings of the twenty-first international conference on
  Machine learning}, page~49. ACM, 2004.

\bibitem{joachims2005support}
Thorsten Joachims.
\newblock A support vector method for multivariate performance measures.
\newblock In {\em Proceedings of the 22nd international conference on Machine
  learning}, pages 377--384. ACM, 2005.

\bibitem{kar2014online}
Purushottam Kar, Harikrishna Narasimhan, and Prateek Jain.
\newblock Online and stochastic gradient methods for non-decomposable loss
  functions.
\newblock In {\em Advances in Neural Information Processing Systems}, pages
  694--702, 2014.

\bibitem{metzler2005markov}
Donald Metzler and W~Bruce Croft.
\newblock A markov random field model for term dependencies.
\newblock In {\em Proceedings of the 28th annual international ACM SIGIR
  conference on Research and development in information retrieval}, pages
  472--479. ACM, 2005.

\bibitem{metzler2005direct}
Donald~A Metzler, W~Bruce Croft, and Andrew McCallum.
\newblock Direct maximization of rank-based metrics for information retrieval.
\newblock {\em unpublished}, 2005.

\bibitem{mohapatra2014efficient}
Pritish Mohapatra, CV~Jawahar, and M~Pawan Kumar.
\newblock Efficient optimization for average precision svm.
\newblock In {\em Advances in Neural Information Processing Systems}, pages
  2312--2320, 2014.

\bibitem{nan2012optimizing}
Ye~Nan, Kian~Ming Chai, Wee~Sun Lee, and Hai~Leong Chieu.
\newblock Optimizing f-measure: A tale of two approaches.
\newblock {\em arXiv preprint arXiv:1206.4625}, 2012.

\bibitem{nedic2009subgradient}
Angelia Nedi{\'c} and Asuman Ozdaglar.
\newblock Subgradient methods for saddle-point problems.
\newblock {\em Journal of optimization theory and applications},
  142(1):205--228, 2009.

\bibitem{parambath2014optimizing}
Shameem~Puthiya Parambath, Nicolas Usunier, and Yves Grandvalet.
\newblock Optimizing f-measures by cost-sensitive classification.
\newblock In {\em Advances in Neural Information Processing Systems}, pages
  2123--2131, 2014.

\bibitem{quoc2007learning}
C~Quoc and Viet Le.
\newblock Learning to rank with nonsmooth cost functions.
\newblock {\em Proceedings of the Advances in Neural Information Processing
  Systems}, 19:193--200, 2007.

\bibitem{rakotomamonjy2004optimizing}
Alain Rakotomamonjy.
\newblock Optimizing area under roc curve with svms.
\newblock In {\em ROCAI}, pages 71--80, 2004.

\bibitem{rennie2005smooth}
Jason~DM Rennie.
\newblock Smooth hinge classification.
\newblock {\em Proceeding of Massachusetts Institute of Technology}, 2005.

\bibitem{ILSVRC15}
Olga Russakovsky, Jia Deng, Hao Su, Jonathan Krause, Sanjeev Satheesh, Sean Ma,
  Zhiheng Huang, Andrej Karpathy, Aditya Khosla, Michael Bernstein,
  Alexander~C. Berg, and Li~Fei-Fei.
\newblock {ImageNet Large Scale Visual Recognition Challenge}.
\newblock {\em International Journal of Computer Vision (IJCV)},
  115(3):211--252, 2015.

\bibitem{song2015direct}
Yang Song, Alexander~G Schwing, Richard~S Zemel, and Raquel Urtasun.
\newblock Direct loss minimization for training deep neural nets.
\newblock {\em arXiv preprint arXiv:1511.06411}, 2015.

\bibitem{szegedy2015going}
Christian Szegedy, Wei Liu, Yangqing Jia, Pierre Sermanet, Scott Reed, Dragomir
  Anguelov, Dumitru Erhan, Vincent Vanhoucke, and Andrew Rabinovich.
\newblock Going deeper with convolutions.
\newblock In {\em Proceedings of the IEEE Conference on Computer Vision and
  Pattern Recognition}, pages 1--9, 2015.

\bibitem{tsochantaridis2005large}
Ioannis Tsochantaridis, Thorsten Joachims, Thomas Hofmann, and Yasemin Altun.
\newblock Large margin methods for structured and interdependent output
  variables.
\newblock In {\em Journal of Machine Learning Research}, pages 1453--1484,
  2005.

\bibitem{van1979information}
CJ~Van~Rijsbergen.
\newblock Information retrieval. dept. of computer science, university of
  glasgow.
\newblock {\em URL: citeseer. ist. psu. edu/vanrijsbergen79information. html},
  1979.

\bibitem{yue2007support}
Yisong Yue, Thomas Finley, Filip Radlinski, and Thorsten Joachims.
\newblock A support vector method for optimizing average precision.
\newblock In {\em Proceedings of the 30th annual international ACM SIGIR
  conference on Research and development in information retrieval}, pages
  271--278. ACM, 2007.

\end{thebibliography}
